\documentclass[11pt, a4paper, reqno]{amsart}

\usepackage{amscd,amsmath,amstext,amsthm,amssymb,amsopn,amsxtra,amsfonts}

\usepackage{paralist}
\usepackage[mathcal]{euscript}
\usepackage{bbold}

\usepackage[english]{babel}

\usepackage[svgnames]{xcolor}

\usepackage{booktabs}
\usepackage{setspace}
\usepackage{tikz}

\usepackage{hyperref}
\hypersetup{
    colorlinks = true,
    linkcolor = RoyalBlue,
    citecolor = Green,    
}

\makeatletter
\g@addto@macro\normalsize{%
  \setlength\abovedisplayskip{10pt}
  \setlength\belowdisplayskip{10pt}
  \setlength\abovedisplayshortskip{5pt}
  \setlength\belowdisplayshortskip{8pt}
}

\newtheorem{thm}{Theorem}

\theoremstyle{definition}
\newtheorem{rmk}[thm]{Remark}

\usepackage{fancyhdr}
\begin{document}
\allowdisplaybreaks
$ $
\vspace{-20pt}

\title{An in-depth look at approximation\\[2pt]via deep and narrow neural networks}

\author{Joris Dommel$^{1,3}$ and Sven A.\ Wegner$^{2,3}$}

\renewcommand{\thefootnote}{}
\hspace{-1000pt}\footnote{\hspace{5.5pt}2020 \emph{Mathematics Subject Classification}: Primary 68T07; Secondary 41A30.\vspace{1.6pt}}

\hspace{-1000pt}\footnote{\hspace{5.5pt}\emph{Key words and phrases}: Deep and narrow neural networks; expressivity; regression tasks; dead neurons. \vspace{1.6pt}}

\hspace{-1000pt}\footnote{\hspace{0pt}$^{1}$\,University of Hamburg, Department of Mathematics, Bundesstra\ss{}e 55, 20146 Hamburg, Germany.\vspace{3pt}}

\hspace{-1000pt}\footnote{\hspace{0pt}$^{2}$\,Corresponding author: University of Hamburg, Department of Mathematics, Bundesstra\ss{}e 55, 20146 Ham-\linebreak\phantom{x}\hspace{1.2pt}burg, Germany, phone:+49\,(0)\,40\:42838\:51\:20, e-mail: sven.wegner@uni-hamburg.de.\vspace{3pt}}

\hspace{-1000pt}\footnote{\hspace{0pt}$^{3}$\,Both authors are thankful for support from the Open Access Publication Fund of Universit\"at Hamburg.}

\begin{abstract}In 2017, Hanin and Sellke showed that the class of arbitrarily deep, real-valued, feed-forward and ReLU-activated networks of width $w$ forms a dense subset of the space of continuous functions on $\mathbb{R}^n$, with respect to the topology of uniform convergence on compact sets, if and only if $w>n$ holds. To show the necessity, a concrete counterexample function $f\colon\mathbb{R}^n\rightarrow\mathbb{R}$ was used. In this note we actually approximate this very $f$ by neural networks in the two cases $w=n$ and $w=n+1$ around the aforementioned threshold. We study how the approximation quality behaves if we vary the depth and what effect (spoiler alert: dying neurons) cause that behavior.
\end{abstract}

\maketitle

\vspace{-15pt}

\section{Introduction}\label{SEC-Intro}

While one in retrospect may point to several precursors for the looming success of neural networks, such as Rosenblatt's perceptron \cite{Rosen} in the 1950s or the expressivity results established around 1990 by Cybenko, Hornik et al.\ \cite{Cyb89, Hornik91, HSW}, it took until the 2010s for neural networks really to hit their stride, cf.~also Goodfellow et al.~\cite[Section 1.2]{Goodfellow} and Prince \cite[p.~ix]{UDL}. One of the final triggers was the discovery that, in practical applications, a deep neural network usually gets by with far fewer neurons than its shallow siblings when challenged with the same regression or classification task. This effect is often referred to as the \emph{power of depth}. There are several results, see e.g.\ Hanin \cite{Hanin} and the references therein, showing that arbitrarily deep ReLU-activated networks form a dense subset of the space of continuous functions $f\colon\mathbb{R}^n\rightarrow\mathbb{R}$ if one allows their width $w$ to be at least $n+1$ (assuming a certain design of the output layer, see Section \ref{SEC-Back}). That this width condition is crucial was made explicit in \cite{HS17} where Hanin and Sellke showed that choosing $w$ less or equal to $n$ leads to non-density. Paluzo-Hidalgo et al.\ \cite{PGG20} recently coined this effect the \emph{power of width}.

\smallskip

The landmark results just mentioned can be further refined in various ways: Hertrich et al.\ \cite{HBDS} study how the class of functions represented by deep networks grows as one increases the depth; Beise et al.\ \cite{Beise21} prove that ReLU-networks with width less or equal to $n+1$ have necessarily unbounded decision regions;  Bresler, Nagaraj \cite{BN20} consider special classes of functions and study the approximation quality via networks in terms of their width, while the function class depends on the depth; Liu, Liang \cite{LL21} consider univariate convex functions and give precise estimates of the approximation quality that can (theoretically!) be achieved.

\smallskip

In addition to the question of expressivity, in practice one is faced with the task of concretely picking a network architecture and calculating weights and bias. Also in this direction there are many recent results, e.g.: Calvo-Pardo et al.\ \cite{CMO23} discuss how to choose an appropriate architecture within the bounds given above; Welper \cite{Welper22} proves that if there is any algorithm that trains a network well, then there is an extension of the network that can be trained by gradient descent to achieve the same approximation quality; Zou et al.\ \cite{ZCZG20} show that gradient descent can find a global minimizer of the training loss function for an appropriate class of deep ReLU-networks. On the other hand, in the aforementioned article \cite{LL21}, Liu, Liang show via examples that in their setting the theoretically possible (and provable!) accuracy cannot be achieved by typical training methods, i.e., gradient descent based algorithms, in practice.

\smallskip

In this paper we return to Hanin and Sellke's results on expressivity \cite{HS17}. Building on their proof for non-density of deep and narrow networks\hspace{1pt}---\hspace{1pt}which uses a concrete counterexample function $f\colon K\subseteq\mathbb{R}^n\rightarrow\mathbb{R}$ and then bounds $\|f-N\|_{K,\infty}$ from below for all networks $N$ of width $w=n$ by a constant independent of the networks depth's $d$\hspace{1pt}---\hspace{1pt}we are mainly interested in the following three questions:\vspace{4pt}

\begin{compactitem}

\item[Q1.] Is the lower bound for $\|f-N\|_{K,\infty}$ in Hanin, Sellke's proof sharp and can it be attained by training neural networks with the standard algorithms?

\vspace{2pt}

\item[Q2.] Despite the non-density, does increasing the depth has an effect on the approximation and what exactly happens in the network during training? 

\vspace{2pt}

\item[Q3.] With width $w=n+1$ the theoretical results yield density\hspace{1pt}---\hspace{1pt}but of course only if arbitrary depth $d$ is allowed. How does the approximation of $f$ in practice depend on the concrete value of $d$?

\end{compactitem}

\vspace{4pt}

\section{Theoretical background}\label{SEC-Back}

Let $n,m\in\mathbb{N}$. In this article, by a \emph{neural network} of width $w\in\mathbb{N}$ and depth $d\in\mathbb{N}$, we shall understand a function $N\colon\mathbb{R}^n\rightarrow\mathbb{R}$ of the form
\begin{equation}\label{EQN}
N= A_d\circ \operatorname{ReLU} \circ A_{d-1} \circ \dots \circ \operatorname{ReLU} \circ A_1
\end{equation}
where $A_i\colon\mathbb{R}^{w_i}\to \mathbb{R}^{w_{i+1}}$ are affine linear maps for $w_1,w_2,\dots,w_{d}\in\mathbb{N}$ with $w_1=n$, $w_d=1$, $w=\max_{i=1,\dots,d}w_i$ and $\operatorname{ReLU}\colon \mathbb{R}^{w_i} \to \mathbb{R}^{w_i}$, $\operatorname{ReLU}(x_1,\dots,x_{w_i})=(\max(0,x_1),\dots \max(0,x_{w_i}))$ is the coordinate-wise Rectified Linear Unit. Notice that many authors do not allow a bias in the output layer but that we, following \cite{HS17}, do so. We denote by $\mathcal{D}^{\operatorname{ReLU},w,d}(\mathbb{R}^n)$ the set of all functions $N$ of the above form. Consequently,
$$
\mathcal{D}^{\operatorname{ReLU},w}(\mathbb{R}^n)=\raisebox{1pt}{$\displaystyle\mathop{\textstyle\bigcup}_{d\in \mathbb{N}}$} \mathcal{D}^{\operatorname{ReLU},w,d}(\mathbb{R}^n)
$$
is the set of neural networks of width $w$ and arbitrary depth. By $\mathcal{C}(\mathbb{R}^n)$ we denote the space of real-valued continuous functions endowed with the topology of uniform convergence on compact subsets. This space is the natural framework to study the expressivity of neural networks, cf.~\cite[Chapter 16]{Wegner24}. Specifically, for the class of ReLU-activated feed-forward neural networks as defined above, \cite[Theorem 1]{HS17} states that $\mathcal{D}^{\operatorname{ReLU},w}(\mathbb{R}^n)\subseteq\mathcal{C}(\mathbb{R}^n)$ is dense if $w\geqslant n+1$ holds. By \cite[Section 3]{HS17} the same condition is also necessary for density. 

\smallskip
   
For the convenience of the reader we give a streamlined proof of the aforementioned necessity, in particular as we are going to study the counterexample employed by Hanin and Sellke further in Sections \ref{SEC-Ex} and \ref{SEC-3} below.

\begin{thm}\label{HS-thm}\cite[Theorem 1.1 and Section 3]{HS17} For $n\in\mathbb{N}$ the space $\mathcal{D}^{\operatorname{ReLU},n}(\mathbb{R}^n)$ of neural networks is not dense in $\mathcal{C}(\mathbb{R}^n)$; indeed it even holds:
$$
\exists\:\eta>0\;\forall\:n\in\mathbb{N}\;\exists\:K\subseteq\mathbb{R}^n\,\text{compact},\,f\in\mathcal{C}(\mathbb{R}^n)\;\forall\,N\in\mathcal{D}^{\operatorname{ReLU},n}(\mathbb{R}^n)\colon\|f-N\|_{K,\infty}\geqslant\eta.
$$
\end{thm}
\begin{proof} Put $\eta=1/16$ and let $n\in\mathbb{N}$ be given. Put $K=\overline{\operatorname{B}}_{1/2}(1/2,\dots,1/2)$, consider
\begin{equation*}
f\colon \mathbb{R}^n\to\mathbb{R},\;\;f(x_1,\dots,x_n)=\mathop{\scalebox{1.15}{$\sum$}}_{i=1}^n \big(x_i-1/2\big)^2
\end{equation*}
and assume that there is $N\in\mathcal{D}^{\operatorname{ReLU},n}(\mathbb{R}^n)$ with $\|f-N\|_{K,\infty}<\eta$. Put $x_0=(1/2,\dots,1/2)$, i.e., $f(x_0)=0$ and put $C=\overline{\operatorname{B}}_{1/\sqrt{8}}(x_0)$. Observe $\partial K=f^{-1}(1/4)$, $\partial C=f^{-1}(1/8)$ and $x_0\in C\subseteq K$. In particular $1/8-1/16\leqslant N|_{\partial C}\leqslant1/8+1/16$. By definition, the function $N$ is of the form \eqref{EQN}. Let 
$$
S_N=\bigl\{x\in\mathbb{R}^n\:\big|\: \forall\:k=1,\dots,d\hspace{1.25pt}{-}\hspace{0.5pt}1\colon \bigl[\operatorname{ReLU} \circ A_{k} \circ\dots \circ \operatorname{ReLU} \circ A_1\bigr](x)>0\text{ coordinate-wise}\bigr\}
$$
which is a convex set and by construction the neural network $N$ coincides with the affine linear map $A_d\circ\cdots\circ A_1$ when restricted to $S_N$. 
We treat two cases.

\smallskip

1. Assume $\partial C\subseteq S_N$ and thus due to convexity $C\subseteq S_N$. Then $N$ is affine linear on $C$ and thus attains its minimum on $\partial C$. Employing that $x_0\in C$ holds, this yields 
$$
\|f-N\|_{K,\infty}\geqslant |f(x_0)-N(x_0)|\geqslant N(x_0)\geqslant\min_{x\in\partial C}N(x)\geqslant1/16.\vspace{-3pt}
$$
\indent{}2. Assume that there exists $x_1\in \partial C\backslash S_N$. Then one can show by induction \cite[Lemma 6]{HS17} that the set $N^{-1}(x_1)$ contains a continuous path to some $x_2\in\mathbb{R}^n\backslash{}K$. As $x_1\in C\subseteq K$ holds, there exists a point $x_3\in\partial K$ with $N(x_1)=N(x_3)$ and thus
$$
\|f-N\|_{K,\infty}\geqslant |f(x_3)-N(x_3)|= |1/4-N(x_1)| \geqslant 1/4-(1/8+1/16)=1/16.
$$
\noindent{}Thus we reach a contradiction in both cases.
\end{proof}

Let us mention that, by Theorem \ref{HS-thm}, any function $h\colon\mathbb{R}^n\rightarrow\mathbb{R}$ of the form $h=g+b$ with $g\in\mathcal{C}(\mathbb{R}^n)$ satisfying $\|f-g\|_{\infty}<\eta/2$, and $b\in\mathbb{R}$ being an arbitrary constant, cannot be approximated with an error smaller than $\eta/2$ by $N\in\mathcal{D}^{\operatorname{ReLU},n}(\mathbb{R}^n)$. Before we continue let us point out the following observation.

\begin{rmk}\label{RMK} For the constant neural network $N_0\colon\mathbb{R}^n\rightarrow\mathbb{R}$, $N_0\equiv1/8$, and $f\colon K\rightarrow\mathbb{R}$ as in the proof of Theorem \ref{HS-thm}, we have $\|f-N_0\|_{K,\infty}=1/8$.
\end{rmk}

\smallskip

\section{Experiments in the non-dense case}\label{SEC-Ex}\vspace{-3pt}

\begin{center}
\small\sc\phantomsection{3.1 Low dimension}\label{SEC-31}
\end{center}

\medskip

After having recalled in Section \ref{SEC-Back} that $\mathcal{D}^{\operatorname{ReLU},n}(\mathbb{R}^n)\subseteq\mathcal{C}(\mathbb{R}^n)$ is \emph{not} dense, we now nevertheless aim to approximate the function from the proof of Theorem \ref{HS-thm}, i.e.,
\begin{equation}\label{FF}
f\colon K=\overline{\operatorname{B}}_{1/2}(1/2,\dots,1/2)\rightarrow\mathbb{R},\;f(x_1,\dots,x_n)=\mathop{\scalebox{1.15}{$\sum$}}_{i=1}^n \big(x_i-1/2\big)^2
\end{equation}
via neural networks $N\in\mathcal{D}^{\operatorname{ReLU},n,d}(\mathbb{R}^n)$ of fixed but variable depth. W.l.o.g.\ we will use a network architecture in which all hidden layers have width $n$, cf.~the pictures in Figures \hyperref[FIG-4]{4}, \hyperref[FIG-5]{5} and \hyperref[FIG-7]{7} as well as equation \eqref{EQN}. As training data we use points $(x,f(x))$, where $x$ comes from a finite subset of $K$. In dimension $n=2$ we generated the latter set in two ways: We either chose uniformly at random from $[0,1]^n$ or we considered a regular grid inside $[0,1]^n$; in both cases we then intersected with $K$. For $n=2$ we indeed did not detect a qualitative difference between both approaches in view of our questions Q1\,--\,Q3. Therefore we will present below only results from the first method. Later, in dimension $n=5$, the so-called curse of high dimensionality will force us to pick the sample in a more sophisticated way, see Section \hyperref[SEC-32]{3.2}.

\smallskip

We start now with the case $n=2$ and a $100{\times}100$ grid in $[0,1]^2$, which amounts to $m=7668$ points after intersecting with the ball $K$. The training process for our neural networks was carried out via \texttt{TensorFlow}, see \cite{dw25} for our complete code. Following \cite{KB17} we employed the \texttt{Adam} optimizer with batch size 1 over 50 epochs of training. We trained ten different models with the same training data. As loss function we used the mean squared error
$$
\operatorname{L}(f,N)={\textstyle\frac{1}{m}}\mathop{\scalebox{1.15}{$\sum$}}_{j=1}^{m}(f(x_j)-N(x_j))^2 \,\text{ with samples }\, x_1,\dots, x_m\in K.
$$
Since we are in fact interested in an approximation with respect to the supremum norm, the following plots show both the MSE loss and the $\infty$-norm. The latter we calculated by maximizing over the $x_1,\dots,x_m$. For each depth we plotted the training processes with the best and the worst approximation result in supremum norm as well as the average over the ten networks that we trained.

\smallskip

\hspace{-20pt}\begin{tabular}{ccc}
\includegraphics[width=155pt]{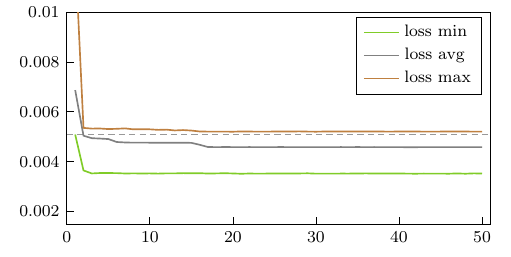}
&
\hspace{-15pt}\includegraphics[width=155pt]{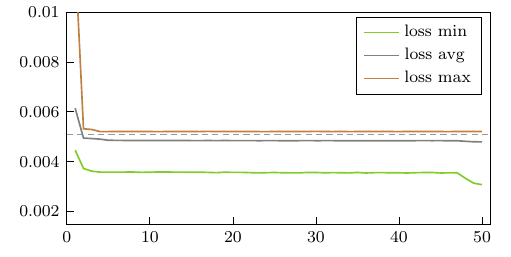}
&
\hspace{-15pt}\includegraphics[width=155pt]{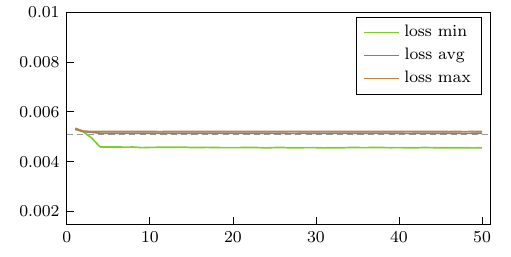}
\end{tabular}

\hspace{-20pt}\begin{tabular}{ccc}
\includegraphics[width=155pt]{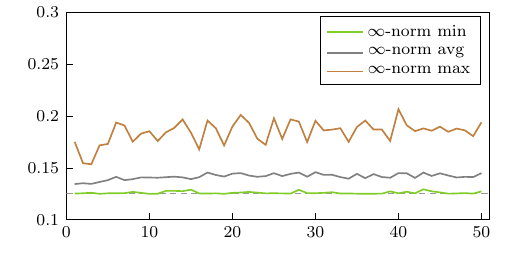}
&
\hspace{-15pt}\includegraphics[width=155pt]{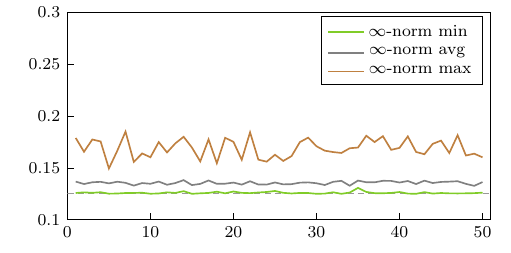}
&
\hspace{-15pt}\includegraphics[width=155pt]{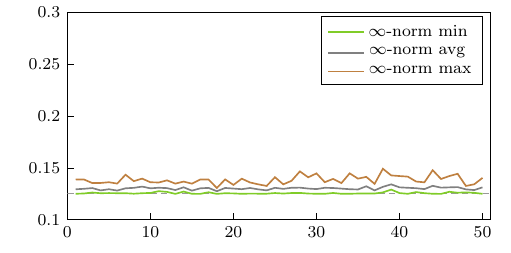}
\end{tabular}\vspace{-5pt}
\begin{center}
\footnotesize

\phantomsection{\bf Figure~1:}\label{FIG-1} Loss function for $n=w=2$ and error in supremum norm during 50 epochs of \\training; from left to right we see the results for $d=1,2,8$; the dashed line corresponds to $N_0$.

\normalsize
\end{center}

\smallskip

While we know $\|f-N\|_{\infty}\geqslant1/16$ from Theorem \ref{HS-thm}, in the experiments $\|f-N\|_{\infty}\geqslant1/8$ holds even for the best performing model, which suggests that, in supremum norm, the constant function $N_0$ from Remark \ref{RMK} is the best approximant! This may be surprising at first but indeed it is very reasonable: The following picture suggests that one ReLU alone cannot have a smaller distance to $f$ in $\infty$-norm than $N_0$ and even if we superponate with a second ReLU this distance cannot be reduced:

\begin{center}

\includegraphics[width=160pt]{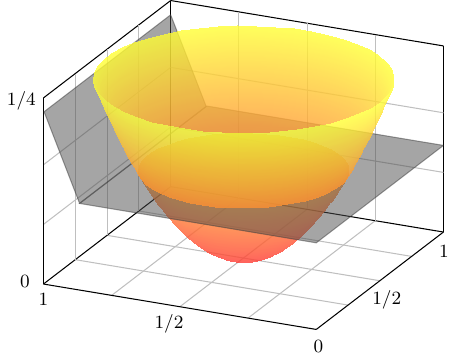}

\footnotesize

\phantomsection{\bf Figure~2:}\label{FIG-2} Graphs of $f$ as in \eqref{FF} and $N=\operatorname{ReLU}\bigl(\langle\left[\begin{smallmatrix}-0.92\\0\end{smallmatrix}\right],\cdot\,\rangle{}+0.12\bigr)+0.12$;\\observe that we only approximate on $K=\overline{\operatorname{B}}_{1/2}(1/2,1/2)$.

\normalsize
\end{center}

\smallskip

The above heuristics applies only for $n=2$. However, independent of the dimension we see that loosing a factor of 2 compared to the theoretical lower bound makes sense in view of the proof of Theorem \ref{HS-thm}, where one estimates $1/8-1/16\leqslant N|_{\partial C}\leqslant 1/8+1/16$ and then uses both estimates in due course.

\smallskip

The next picture shows the set $\operatorname{argmax}_{x\in K}|f(x)-N(x)|$ for $n=2$ and our ten trained networks. We see that the maximizers indeed occur as in the two cases treated in the proof of Theorem \ref{HS-thm}. Notice that we plot below only the maximizers found within $\{x_1,\dots,x_m\}$; for a constant network [a network as in Figure \hyperref[FIG-2]{2}] obviously all [almost all] points in $\partial{}K$ will be maximizers.

\begin{center}
\begin{tabular}{cccc}
\phantom{FFFF}&\includegraphics[height=100pt]{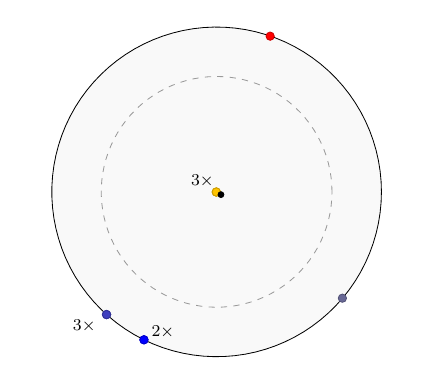}
&
\hspace{-15pt}\includegraphics[height=100pt]{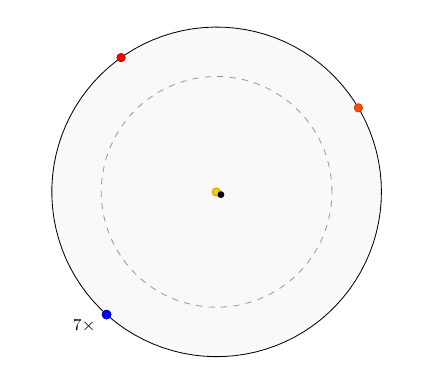}
&
\hspace{-15pt}\includegraphics[height=100pt]{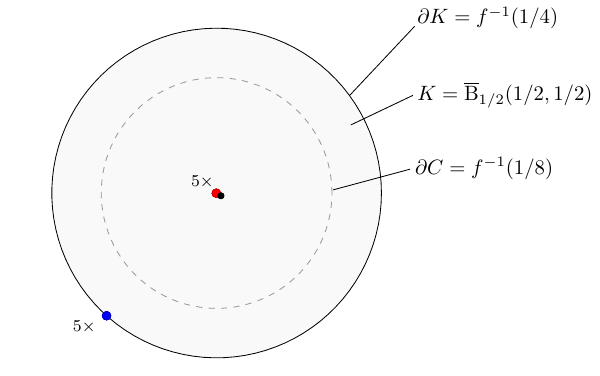}
\end{tabular}

\footnotesize

\phantomsection{\bf Figure~3:}\label{FIG-3} Maximizers of $|f{-}N|$ within $\{x_1,\dots,x_m\}$ for ten neural\\networks; from left to right we see the depths 1, 2 and 8.

\normalsize

\end{center}

\smallskip

Going back to Figure \hyperref[FIG-1]{1}, we like to point out three more things: Firstly, one can see that the average approximation with respect to the supremum norm is considerably better in depth 8 than it is in depth 1 or 2. This suggests that, despite Theorem \ref{HS-thm}, increasing the depth of a narrow network may yield a gain in approximation quality. Secondly, the approximation in terms of the loss function behaves exactly opposite, i.e., $\operatorname{L}(f,N)$ in average is larger for depth 8 than it is for depth 1 and 2. This is at first sight doubly surprising since our algorithm does seek to minimize the loss\hspace{1pt}---\hspace{1pt}not the supremum norm\hspace{1pt}---\hspace{1pt}and further as a non-constant approximant as shown in Figure \hyperref[FIG-2]{2} can be expressed equally good in any depth and should produce a smaller loss than the constant function while it has the same distance in $\infty$-norm to $f$. Thirdly, in depth 1 and 2 the gap between the worst and the best model is considerably higher than in depth 8. This suggests that in these cases it happens more often that the training process gets stuck in a local minimum of the loss function.

\smallskip

Let us now take a look inside the trained neural networks. We start with depth 1 and print out below the weights and bias after training has been completed in three examples. The three networks in Figure \hyperref[FIG-4]{4} correspond to the best (left), an about average (middle) and the worst (right) models; everything with respect to supremum norm:

\begin{center}
\begin{tabular}{ccc}
\includegraphics[height=70pt]{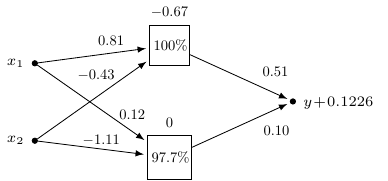}
&
\hspace{-20pt}\includegraphics[height=70pt]{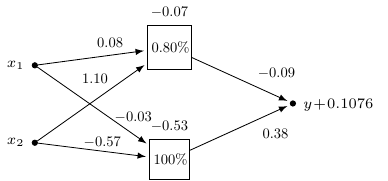}
&
\hspace{-20pt}\includegraphics[height=70pt]{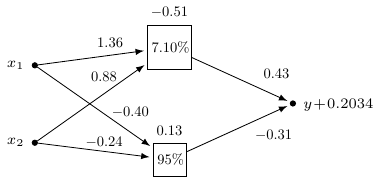}
\end{tabular}

\vspace{3pt}

\footnotesize

\phantomsection{\bf Figure~4:}\label{FIG-4} Weights and bias for three networks $N_1$ (left), $N_2$ (middle) and $N_3$ (right) for $n=w=2$\\and $d=1$. The percentages indicate how many of the sample points in $K$ get mapped to zero.

\normalsize

\end{center}

Looking at the weights and bias of $N_1$ we see immediately that the two neurons will both map (almost) the whole ball $K$ to zero; $N_1$'s output will thus be almost constant $0.1226$. Especially in cases as the upper neuron of $N_1$, which maps 100\% to zero, it is common to call such neurons \emph{dead} or \emph{inactive}, see Lee et al.\ \cite{LKYC24}, Nair, Hinton \cite[Section 5.1]{NH}, Teso-Fz-Beto{\~n}o et al.\ \cite{BZCOBU22} and Lu et al.\ \cite{LSSK20} for details. We point out in particular, that once a neuron is dead it cannot be resuscitated if one uses algorithms akin to gradient descent. In our case, out of the ten depth 1 models, a total of eight had at least one dead neuron amounting qualitatively to a function as in Figure \hyperref[FIG-2]{2}. There was one model with two dead neurons representing a constant function. Only in $N_3$ both neurons were still alive at the end of training, which suggests a picture of two superponated ReLUs. In terms of approximation quality, we have
\begin{equation*}
\begin{array}{rclcrclcrcl}
\|f-N_1\|_{K,\infty}&\hspace{-5pt}=&\hspace{-5pt}0.1274, && \|f-N_2\|_{K,\infty} &\hspace{-5pt}=&\hspace{-5pt}0.1424,& &\|f-N_3\|_{K,\infty}&\hspace{-5pt}=&\hspace{-5pt}0.1940,\\[5pt]
\operatorname{L}(f,N_1)&\hspace{-5pt}=&\hspace{-5pt}0.0045, &&\operatorname{L}(f,N_2)&\hspace{-5pt}=&\hspace{-5pt}0.0045,&&\operatorname{L}(f,N_3)&\hspace{-5pt}=&\hspace{-5pt}0.0036.\\
\end{array}
\end{equation*}
As mentioned, the almost dead network $N_1$ yields, in $\infty$-norm, the best approximant. We will come back to this effect when we discuss the depth 8 networks below. The worst performing network in the supremum norm, i.e., $N_3$, has as well a neuron which is almost dead and a much larger output bias compared to $N_1$ and $N_2$. In terms of MSE loss, $N_3$ is indeed the best performing among the networks $N_1$, $N_2$ and $N_3$ and thus an example for the MSE loss differing significantly from the supremum norm. 

\smallskip

Let us now consider the case of depth 8 and again look at two examples. Here the effect of dying neurons took place even in more extreme form. Indeed, in all of our experiments each neuron either mapped everything to zero or it mapped nothing to zero.

\medskip

\begin{center}

\hspace{30pt}\includegraphics[height=30pt]{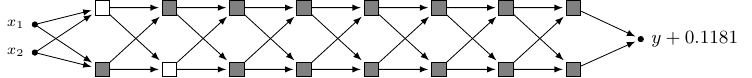}

\vspace{10pt}

\hspace{30pt}\includegraphics[height=30pt]{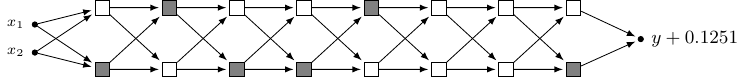}

\vspace{5pt}

\footnotesize

\phantomsection{\bf Figure~5:}\label{FIG-5} Two trained networks $N_4$ (above) and $N_5$ (below) of depth 8.\\Gray neurons map everything to zero while white neurons map nothing to zero.

\normalsize

\end{center}

\smallskip

Besides the aforementioned dichotomy of either dead or `fully alive' neurons the following two cases occurred: Either the dead neurons block the way and we get again a constant function, or a path of fully alive neurons remains that feeds input data through until the output layer. As  Figure \hyperref[FIG-1]{1} already suggested, the approximation quality of all ten trained models in supremum norm was very much alike. For the two networks in Figure \hyperref[FIG-5]{5} we have
\begin{equation*}
\begin{array}{rclcrcl}
\|f-N_4\|_{K,\infty}&\hspace{-5pt}=&\hspace{-5pt}0.1365, && \|f-N_5\|_{K,\infty}&\hspace{-5pt}=&\hspace{-5pt}0.1251,\\[5pt]
\operatorname{L}(f,N_4)&\hspace{-5pt}=&\hspace{-5pt}0.0045, && \operatorname{L}(f,N_5)&\hspace{-5pt}=&\hspace{-5pt}0.0052.
\end{array}
\end{equation*}
That is, the constant network $N_4$ performs worse in supremum norm but it is a better approximant with respect to mean squared error compared to the non-constant network $N_5$. Let us mention that we have $\operatorname{L}(f,N_0)=0.0051$, i.e., the function with constant value $1/8$, for which we have $\|f-N_0\|_{K,\infty}=0.125$, performs, with respect to MSE, worse than $N_4$ but slightly better than $N_5$.

\bigskip

\begin{center}
\small\sc\phantomsection{3.2 Higher dimension}\label{SEC-32}
\end{center}

\smallskip

We now look at $n=5$ and emphasize to begin with, that in Theorem \ref{HS-thm} the constant $\eta=1/16$ is independent of the space dimension $n$. Due to the concentration of measure phenomenon, see \cite[Chapter 9]{Wegner24}, we can neither sample points from $[0,1]^5$ uniformly at random nor can we use a regular grid, as with both methods the majority of points would be located outside the ball $K\subseteq[0,1]^5$. We therefore proceed as suggested in \cite[Section 2.5]{BHK}: We first generate $100\,000$ points from a standard normal distribution and divide each point by its euclidean norm. This yields a set of points on the unit sphere which we scale with randomly drawn $\rho\in[0,1]$ stemming from a distribution with density proportional to the 4th power of the radius. Finally we translate the sample to $K$ by adding $(1/2,\dots,1/2)$.

\smallskip

As we are still dealing with the non-dense case we keep the width at $w=n=5$ and conduct experiments as before for three different depths. We start with $d=1$, but then, in order to keep the ratio of width divided by depth constant, we consider $d=10$ and $d=20$.

\smallskip

\hspace{-20pt}\begin{tabular}{ccc}
\includegraphics[width=155pt]{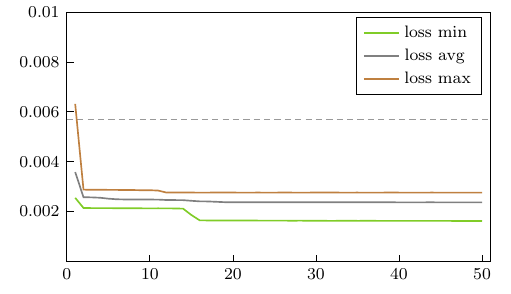}
&
\hspace{-15pt}\includegraphics[width=155pt]{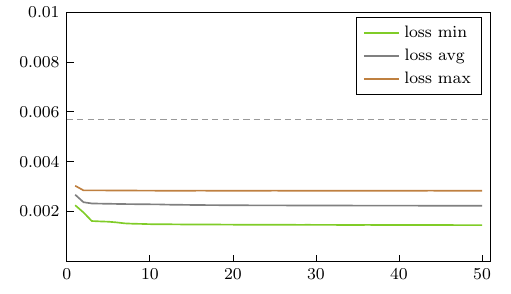}
&
\hspace{-15pt}\includegraphics[width=155pt]{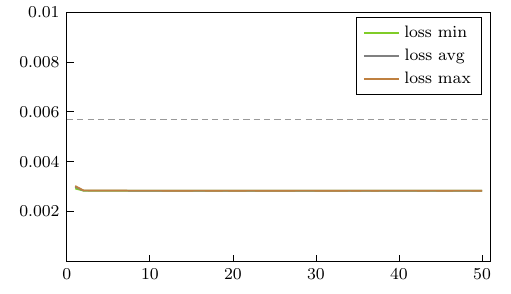}
\end{tabular}

\hspace{-20pt}\begin{tabular}{ccc}
\includegraphics[width=155pt]{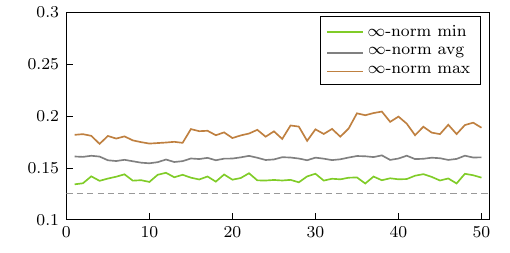}
&
\hspace{-15pt}\includegraphics[width=155pt]{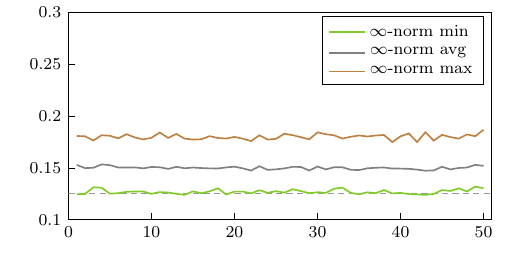}
&
\hspace{-15pt}\includegraphics[width=155pt]{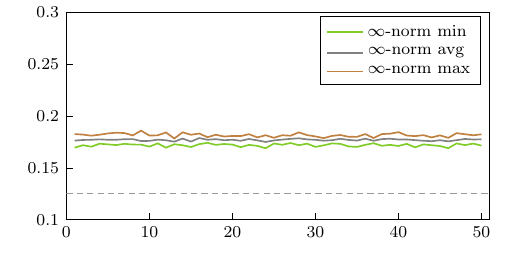}
\end{tabular}
\begin{center}\vspace{-5pt}
\footnotesize

\phantomsection{\bf Figure~6:}\label{FIG-6} Loss function for $n=w=5$ and error in supremum norm during 50 epochs of\\training; from left to right we see the results for $d=1,10,20$; the dashed line corresponds to $N_0$.

\normalsize

\end{center}

\smallskip

Indeed, the plots in Figure \hyperref[FIG-6]{6} display qualitatively the same effects as we detected in dimension two, see Figure \hyperref[FIG-1]{1}\hspace{1pt}---\hspace{1pt}with two exceptions: Firstly, while approximation in supremum norm improves when we increase the depth from 1 to 10, for depth $20$ we again obtain far worse results. Secondly, we see that now, with increased dimension but still with $w=n$ being too small for density, that in terms of the mean squared error indeed for every depth every model performs considerable better than the constant function.

\smallskip

As before let us next look inside one of the trained networks in dimension five. We choose the one which has minimal distance to $f$ in supremum norm. In the following picture we dropped the connecting arrows; it is implied that the network is fully connected. Note that, in contrast to dimension 2, the non-dead neurons were not all `fully alive' but indeed represented proper ReLUs, i.e., neither constant nor affine linear functions.

\vspace{10pt}

\begin{center}

\hspace{30pt}\includegraphics[width=380pt]{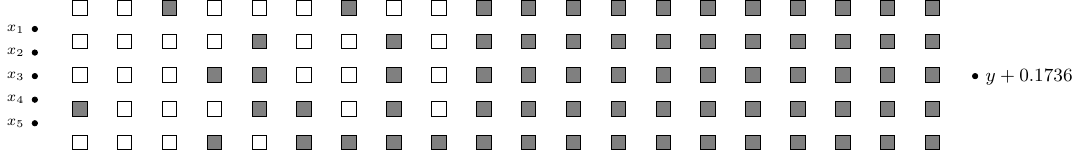}
\footnotesize

\vspace{5pt}

\phantomsection{\bf Figure~7:}\label{FIG-7} The trained network $N_6$ of width $w=5$ and depth $d=20$ is our $\infty$-norm\\best approximation of $f$; gray neurons are dead, white neurons are not dead.

\normalsize

\end{center}

\medskip

The above network achieved $\|f{-}N_6\|_{K,\infty} = 0.1715$ and performed far worse than the constant network $N_0\equiv1/8$. As said already, in terms of the loss function the situation is reversed, concretely we have $\operatorname{L}(f,N_0)=0.0057$ and $\operatorname{L}(f,N_6)=0.0028$. 

\smallskip

Going through our ten trained networks, we observed that the last layer of neurons in the $d=20$ case was in most models completely dead. This corroborates recent results by Lu et al.\ \cite{LSSK20} who proved that for $d\rightarrow\infty$ the whole network will die with high probability. Assuming that a network $N$ has died, the optimization process, seeking to minimize $\operatorname{L}(f,N)$, can only adjust the final bias of the network which is precisely the constant value of $N$. Calculating the loss on different constant functions, e.g., 
$$
\operatorname{L}(f,{\textstyle\frac{1}{16}})=0.0163,\;\;\operatorname{L}(f,{\textstyle\frac{2}{16}})=0.0057,\;\;\operatorname{L}(f,{\textstyle\frac{3}{16}})=0.0029,\;\;\operatorname{L}(f,{\textstyle\frac{4}{16}})=0.0079,
$$
makes it clear why we do not reach the $\infty$-norm minimizer but instead get a network with constant value close to $3/16=0.1875$. A concrete example is the network $N_6$ in Figure \hyperref[FIG-7]{7}.

\medskip

We conclude this section with two complete tables of the best, worst and average approximants in supremum norm for all trained models. As mentioned at the beginning of Section \hyperref[SEC-31]{3.1}, we  carried out the experiments for $n=w=2$ a second time with a training set coming from evaluations of $f$ on a subset of $K$ that was drawn uniformly at random from $[0,1]^2$. For the sake of completeness the corresponding results are included below. Recall, that the right table is based on a sample generated randomly, however not uniformly, but instead in the way explained at the beginning of Section \hyperref[SEC-32]{3.2}.

\medskip

\begin{center}
\begin{minipage}{210pt}\label{TABLE}
\begin{center}\footnotesize
Width 2\vspace{3pt}
\\\begin{tabular}{cccc}
        \toprule
Training set, depth        &min & avg & max  \\
        \midrule
   grid, $d=1$    & 0.1274 & 0.1450 & 0.1940 \\
   grid, $d=2$    & 0.1262 & 0.1364 & 0.1602 \\
   grid, $d=8$    & 0.1251 & 0.1313 & 0.1404 \\[2pt]
   random, $d=1$  & 0.1251 & 0.1465 & 0.1957 \\
   random, $d=2$  & 0.1261 & 0.1536 & 0.2153 \\
   random, $d=8$  & 0.1254 & 0.1286 & 0.1347 \\[-12pt]
                    & \phantom{xxxxxx} & \phantom{xxxxxx}& \phantom{xxxxxx} \\
        \bottomrule
\end{tabular}
\end{center}
\end{minipage}
\begin{minipage}{210pt}
\begin{center}\footnotesize
Width 5\vspace{3pt}
\\\begin{tabular}{cccc}
        \toprule
Training set, depth        &min & avg & max  \\
        \midrule
       &  &  &  \\
   & & & \\
   & & & \\[2pt]
   random, $d=1\phantom{0}$  & 0.1407 & 0.1602 & 0.1887 \\
   random, $d=10$  & 0.1305 & 0.1520 & 0.1868 \\
   random, $d=20$  & 0.1715 & 0.1776 & 0.1824 \\[-12pt]
                  & \phantom{xxxxxx} & \phantom{xxxxxx}& \phantom{xxxxxx} \\
        \bottomrule
\end{tabular}
\end{center}
\end{minipage}

\vspace{4pt}

\footnotesize

\smallskip

\phantomsection{\bf Table 1:} Numerical results for all models considered in Section \ref{SEC-Ex}.

\normalsize

\end{center}

\medskip\smallskip

\section{Experiments in the dense case}\label{SEC-3}

Finally, we want to address the case $w=n+1$ just on the other side of the dense/non-dense threshold. We know then that $\mathcal{D}^{\operatorname{ReLU},n+1}(\mathbb{R}^n)\subseteq\mathcal{C}(\mathbb{R}^n)$ \emph{is} dense, but, of course, in the set on the left hand side the networks may be arbitrarily deep. It is thus a natural question if, for the function $f$ and  the same depths as studied in Section \ref{SEC-Ex}, an improvement in accuracy is already detectable. To answer this question, we again trained ten networks on the sets used in Sections \hyperref[SEC-31]{3.1} and \hyperref[SEC-32]{3.2}, respectively. In both cases, we adjusted the number of training epochs to 100, since there are now more parameters to be optimized.

\medskip

\hspace{-20pt}\begin{tabular}{ccc}
\includegraphics[width=155pt]{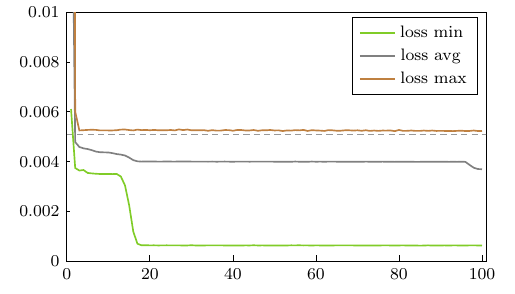}
&
\hspace{-15pt}\includegraphics[width=155pt]{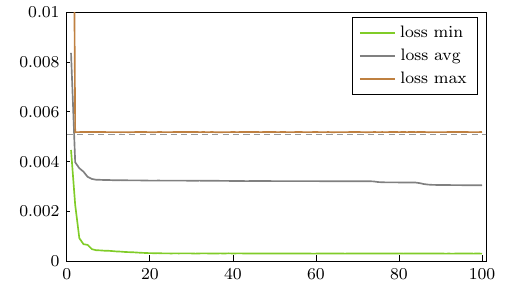}
&
\hspace{-15pt}\includegraphics[width=155pt]{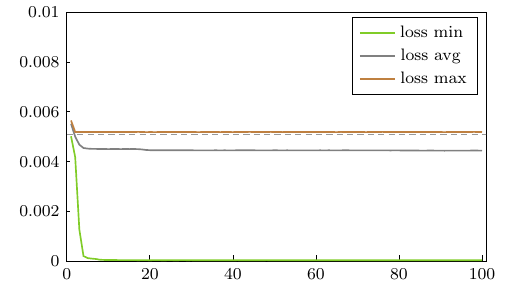}
\end{tabular}

\hspace{-20pt}\begin{tabular}{ccc}
\includegraphics[width=155pt]{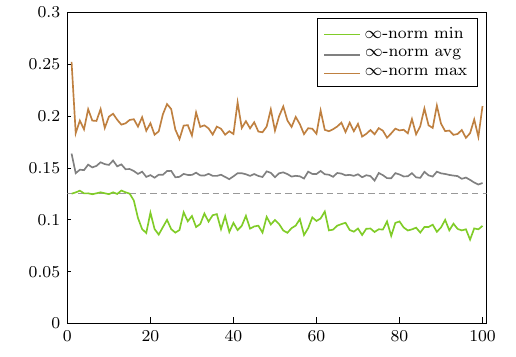}
&
\hspace{-15pt}\includegraphics[width=155pt]{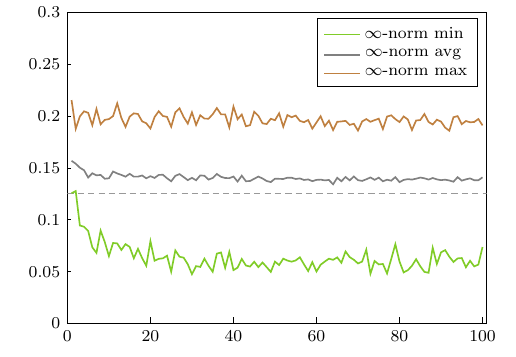}
&
\hspace{-15pt}\includegraphics[width=155pt]{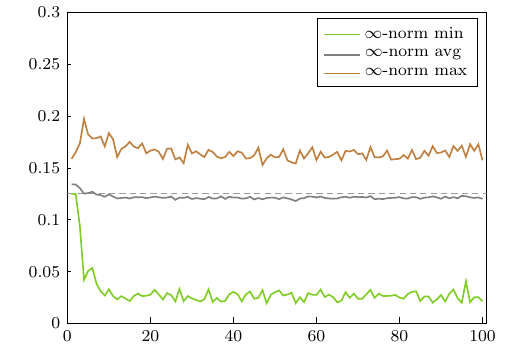}
\end{tabular}
\begin{center}\vspace{-5pt}
\footnotesize

\phantomsection{\bf Figure~8:}\label{FIG-8} Loss function for $n=2$, $w=3$ and error in supremum norm during 100 epochs of \\training; from left to right we see the results for $d=1,2,8$; the dashed line corresponds to $N_0$.

\normalsize

\end{center}

\smallskip

Comparing Figure \hyperref[FIG-8]{8} with Figure \hyperref[FIG-1]{1}, we see that having increased $w$ by just 1 leads already to much better approximations in terms of MSE. Also in $\infty$-norm at least the best performing network now clearly beats the constant approximant\hspace{1pt}---\hspace{1pt}which is perfectly reasonable if we look back at Figure \hyperref[FIG-2]{2} where now, already in depth 1, superponation of three suitably rotated ReLUs is theoretically possible. On the other hand a considerable number of models also here performs worse than the constant function so that we in average beat the latter only by a very small margin and only in depth 8.

\medskip

Figure  \hyperref[FIG-9]{9} shows the dying neuron effect in the current setting. Network $N_7$  corresponds to the green lines in both plots on the right of Figure \hyperref[FIG-8]{8}, while $N_8$ corresponds to the brown lines in the same plots. The pictures suggest that $N_7$ performs so much better as here on the very left a $3\times2$-block of neurons is not dead and its output can be fed through around the dead neurons to the output. In $N_8$ the largest non-dead block is of size $2\times2$. Also the total number of dead neurons in $N_8$ is considerably higher than in $N_7$.

\medskip

\begin{center}

\hspace{30pt}\includegraphics[]{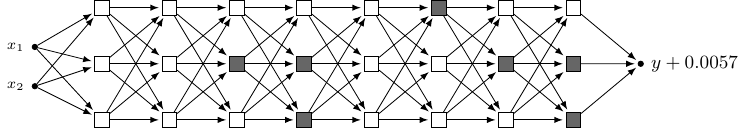}

\vspace{10pt}

\hspace{30pt}\includegraphics[]{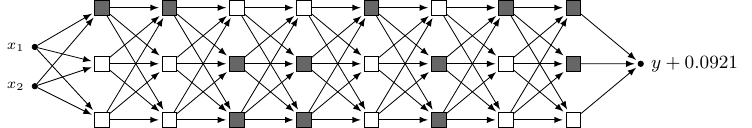}

\vspace{5pt}

\footnotesize

\phantomsection{\bf Figure~9:}\label{FIG-9} Two trained networks $N_7$ (above) and $N_8$ (below) of depth 8.

\normalsize

\end{center}

\smallskip

We continue with input dimension five.

\smallskip

\hspace{-20pt}\begin{tabular}{ccc}
\includegraphics[width=155pt]{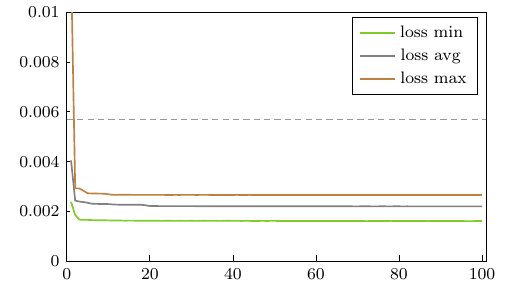}
&
\hspace{-15pt}\includegraphics[width=155pt]{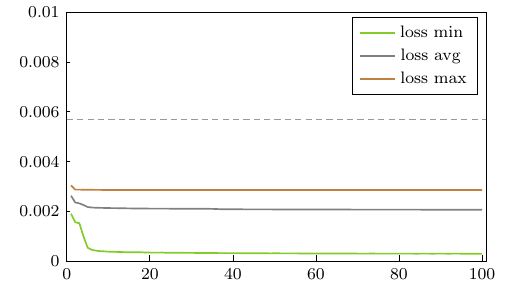}
&
\hspace{-15pt}\includegraphics[width=155pt]{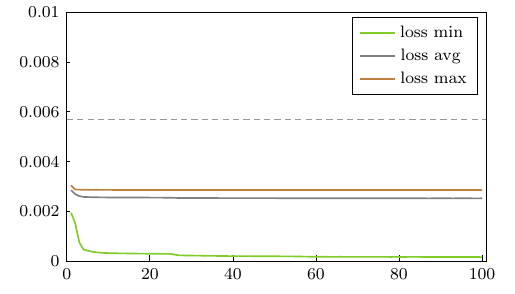}
\end{tabular}

\hspace{-20pt}\begin{tabular}{ccc}
\includegraphics[width=155pt]{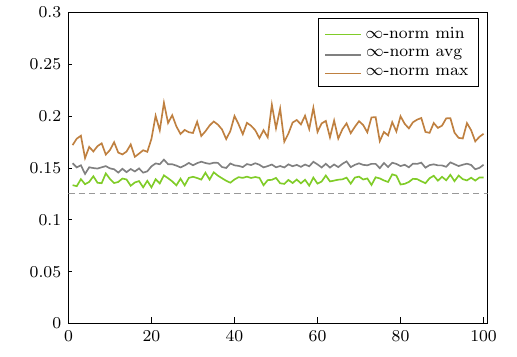}
&
\hspace{-15pt}\includegraphics[width=155pt]{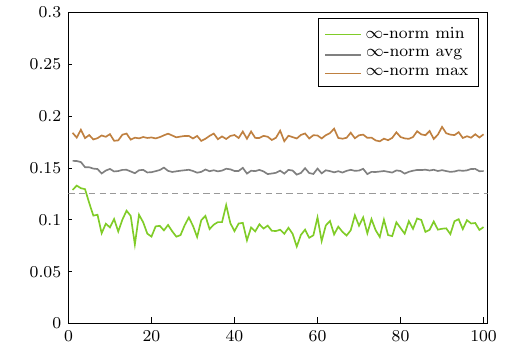}
&
\hspace{-15pt}\includegraphics[width=155pt]{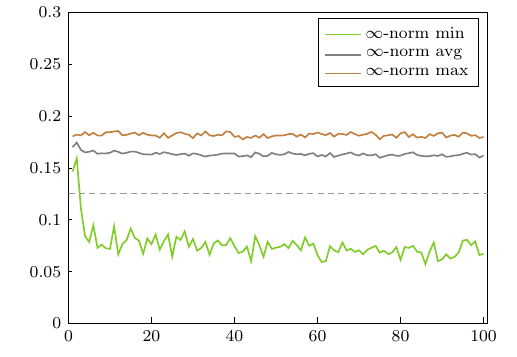}
\end{tabular}
\begin{center}\vspace{-5pt}
\footnotesize

\phantomsection{\bf Figure~10:}\label{FIG-10} Loss function for $n=5$, $w=6$ and error in supremum norm during 100 epochs of \\training; from left to right we see the results for $d=1,10,20$; the dashed line corresponds to $N_0$.

\normalsize

\end{center}

\vspace{3pt}

A comparison of Figure \hyperref[FIG-10]{10} with Figure \hyperref[FIG-6]{6} exhibits a similar improvement in terms of MSE loss. With respect to supremum norm, the best performing networks also here clearly beat $N_0$, while in average we do not achieve the latter. Moreover, while in Figure \hyperref[FIG-8]{8} the average error in supremum norm decays as $d$ grows, Figure \hyperref[FIG-10]{10} shows that the best average performance happens for low depth and that we have a significantly worse error for $d=20$.

\smallskip

Let us again look into some networks. Figure  \hyperref[FIG-1]{11} shows the dying neuron effect for networks $N_9$ and $N_{10}$ which correspond to the green resp.~brown lines in the plot at the right lower corner of Figure \hyperref[FIG-8]{8}. One can see immediately, that in $N_{10}$ the last layer has died completely, which means that $N_{10}$ represents a constant function. In $N_9$, on the other hand, we see that the first layer is completely alive and the second has only one dead neuron. Moreover, there is no layer that has died completely. This combinations seems to be the reason for the much better performance of $N_9$.

\vspace{10pt}

\begin{center}
\hspace{30pt}\includegraphics[width=380pt]{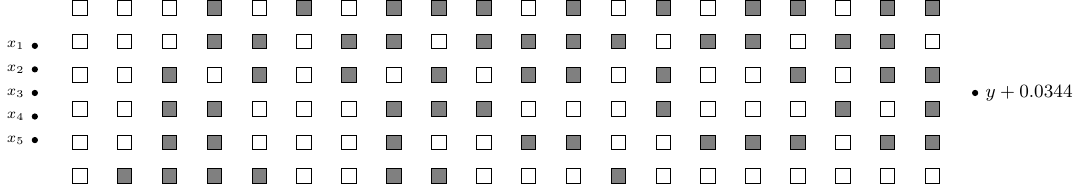}

\vspace{20pt}

\hspace{30pt}\includegraphics[width=380pt]{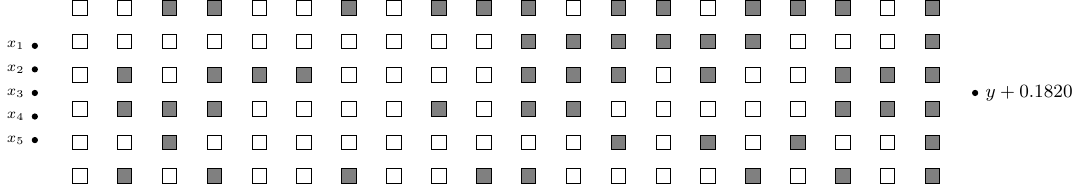}

\footnotesize

\vspace{5pt}

\phantomsection{\bf Figure~11:}\label{FIG-11} The trained networks $N_9$ and $N_{10}$ of width $w=6$\\and depth $d=20$ are our $\infty$-norm best resp.~worst approximation of $f$.

\normalsize

\end{center}

\medskip

As mentioned, the findings of \cite{LSSK20} predict that for $d\rightarrow\infty$ the network will die completely. Although we are still quite far away from this in Figure \hyperref[FIG-11]{11}, the averages in Figure \hyperref[FIG-10]{10} show that we end up in the same situation as discussed in Section \ref{SEC-Ex}, i.e., networks with constant value minimizing the loss function and thus performing worse than $N_0$ in supremum norm.

\section{Conclusion}

Let us summarize our findings and formulate answers to the questions Q1\hspace{0.5pt}--\hspace{1.5pt}Q3 in Section \ref{SEC-Intro}: 

\vspace{4pt}

\begin{compactitem}

\item[A1.] Our experiments suggest that $\min_{N\in\mathcal{D}^{\text{ReLU},n}(\mathbb{R}^n)}\|f-N\|_{K,\infty}=1/8$ holds with the function $N_0\equiv1/8$ being a minimizer for every $n\in\mathbb{N}$. The constant function $N_0$ can easily be expressed by any neural network independent of its depth or width. In our experiments it was attained only in some cases, see A2 below.

\vspace{3pt}

\item[A2.] In our low dimensional experiments, an improvement of approximation quality through an increase of depth was detectable and our models came, even in average, close to the constant function $N_0$ (compare with A1). For higher dimensions, and correspondingly increased depth, our models did not approach $N_0$ but headed for a different constant function. The reason for both is the effect of dying neurons. However, while in low dimensions the discrepancy between minimizing $\operatorname{L}(f,N)$ and minimizing $\|f-N\|_{K,\infty}$ is small, it seems to become larger as the dimension grows. Notice that computing $\operatorname{L}(f,N)$ depends on the choice of a sample, while $\|f-N\|_{K,\infty}$, at least for constant $N$, can be calculated analytically. Therefore the latter might also be related to the concentration of measure phenomenon.

\vspace{3pt}

\item[A3.] In the case $w=n+1$, i.e., the smallest width for which $\mathcal{D}^{\text{ReLU},w}(\mathbb{R}^n)\subseteq\mathcal{C}(\mathbb{R}^n)$ actually is dense, we also achieve noticeably better approximations of the function $f$ if we increase the depth. It seems however that, at least for higher dimensions, there is a threshold for the depth and after that the dying ReLU phenomenon will again lead to constant networks, which brings us back into the same situation as explained in A2.

\end{compactitem}

\medskip

\bibliographystyle{amsplain}
\bibliography{Ref-dl}

\end{document}